\documentclass{article}

\usepackage[utf8]{inputenc} 
\usepackage[T1]{fontenc}    
\usepackage{hyperref}       
\usepackage{url}            
\usepackage{booktabs}       
\usepackage{amsfonts}       
\usepackage{nicefrac}       
\usepackage{microtype}      
\usepackage{xcolor}         

\usepackage{tikz}
\usepackage{xcolor}         
\usepackage{amsthm}
\usepackage{adjustbox}

\usepackage{graphicx}
\usepackage{amsmath,amssymb}
\usepackage{subcaption}
\usepackage{multirow}

\usepackage{caption}
\usepackage{enumitem}

\usepackage{pgfplots}
\usepackage{pgfplotstable}
\usepackage{tikz}
\usetikzlibrary{intersections}
\usetikzlibrary{matrix}
\usepgfplotslibrary{fillbetween}
\usepgfplotslibrary{groupplots}
\usepackage{xcolor}
\definecolor{pastelgreen}{rgb}{0.35, 0.75, 0.35}
\definecolor{pastelred}{rgb}{0.9, 0.3, 0.28}
\definecolor{pastelblue}{rgb}{0.35, 0.55, 0.75}

\usepackage{pifont} 
\newcommand{\cmark}{\textcolor{green!80!black}{\ding{51}}}
\newcommand{\xmark}{\textcolor{red}{\ding{55}}}

\usepackage[round]{natbib}

\usepackage{amsmath}
\DeclareMathOperator{\E}{\mathbb{E}}
\DeclareMathOperator{\tr}{tr}    
\DeclareMathOperator{\fisher}{\mathcal{I}}

\usepackage{hyperref}
\usepackage[capitalise]{cleveref}
\newtheorem{theorem}{Theorem}[section]

\newtheorem{lemma}[theorem]{Lemma}

\newtheorem{definition}[theorem]{Definition}

\newtheorem{remark}[theorem]{Remark}

 \usepackage[preprint]{neurips_2026}

\title{Heteroscedasticity of Denoising 
 Score Matching with Generalised Smooth Noise}

\author{%
  Juyan Zhang$^1$ \And Rhys Newbury$^1$ \And Xinyang Zhang$^2$ \AND
  Trin Tran$^1$ \And Dana Kuli\'c$^1$ \And Michael Burke$^1$ \\[0.2cm]
  $^1$Monash University \quad $^2$Amazon
}

\begin{document}

\maketitle

\begin{abstract}
  Score Matching (SM) is a powerful framework for estimating the log-density derivatives of a distribution without calculating its normalizing constants. This capability has made it a cornerstone across multiple domains, from classical statistical estimation and energy-based models to modern diffusion-based generative models. In practice, these models rely almost exclusively on Denoising Score Matching (DSM) as a tractable proxy for score matching. This ubiquity naturally raises a fundamental question: \emph{Is DSM truly "score matching for free"?} In this work, we demonstrate that DSM is not a perfect substitute. We prove that the denoising objective is inherently heteroscedastic,  the variance of model parameters fluctuates unpredictably based on both noise levels and the underlying data geometry. This instability is baked into the mathematical structure of the DSM. To address this, we derive an ideal weighting function that equalizes this variance, yielding a homoscedastic generalization of DSM. Since the ideal weights are often empirically inaccessible, we show that a practical approximation weighting function via Taylor expansion reduces gradient variance during training, at the cost of statistical optimality. Notably, this provides a theoretical justification for an existing heuristic weight used in Isotropic Gaussian Diffusion. We validate our theory across different perturbed distributions and for higher-order scores.
\end{abstract}

\section{INTRODUCTION}

Score matching, fitting a model to match the gradient of the data distribution rather than the densities themselves, provides a principled framework for estimating arbitrary order log-density derivatives (scores) without requiring access to normalizing constants. However, the true score of the data distribution is often inaccessible. Instead, modern generative modeling relies almost exclusively on Denoising Score Matching (DSM), which replaces the intractable objective with a tractable surrogate defined on noise-perturbed data. This has made it a central tool in modern generative modeling, particularly in energy~\citep{song2021train, swersky2011autoencoders} and diffusion-based models~\citep{song2020score, karras2022elucidating, sohl2015deep, ho2020denoising}.

The foundational work of \cite{vincent2011connection} proved that the expected gradients of the DSM and SM objectives are equivalent, leading to an implicit assumption that DSM is a "free" substitute. However, matching the average gradient is not enough; the variance of the loss affects training stability. Consequently, it remains unclear whether DSM faithfully preserves the statistical structure of score matching across more general settings, including higher-order score estimation and for arbitrary smooth noise distributions.
\begin{figure}
    \centering
    \includegraphics[width=0.75\linewidth, trim=0.2cm 0cm  0cm 0cm, clip]{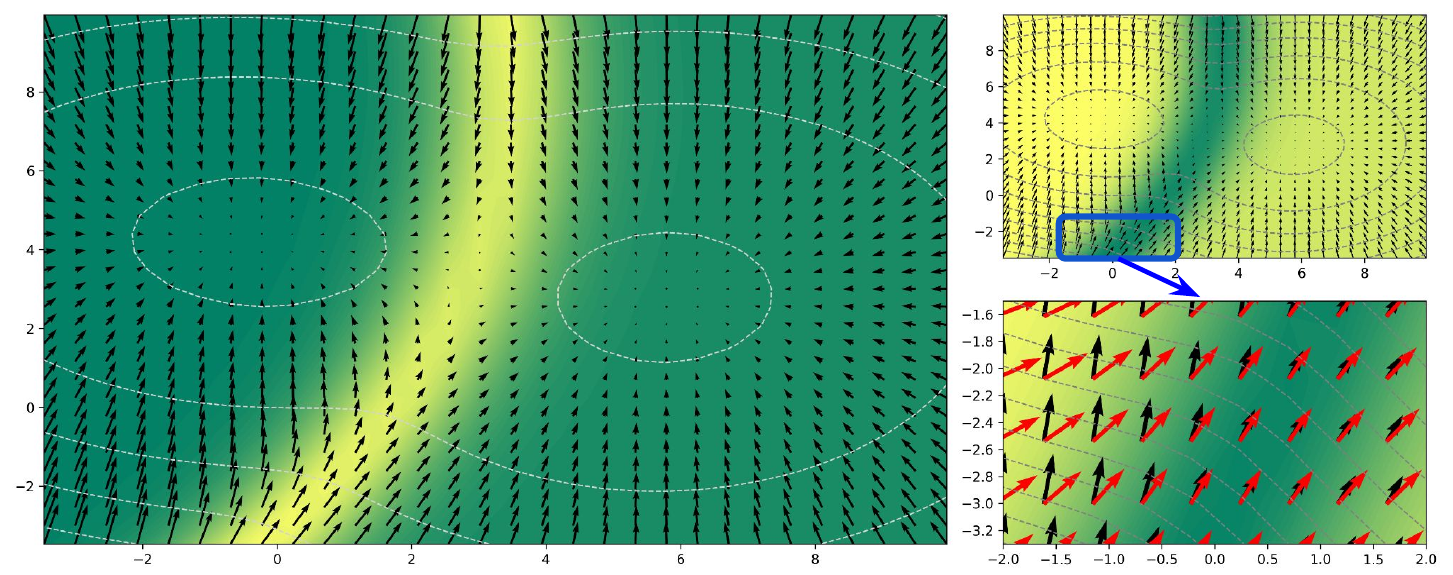}
    \caption{\textbf{Visualizing the heteroscedasticity and effect of Godambe weighting on a 2D Gaussian Mixture.} (Left) Variance of the true score field, with lighter colors indicating regions of higher variance (prominently along the boundary between components); (Top Right) The green region is systematically down-weighted by Godambe weighting; (Bottom Right) Magnified view of the weighting dynamics, comparing the true score field (black arrows) against the dominant eigenvector of the Godambe optimal weighting matrix (red arrows). The dominant eigenvector shows the direction in which the denoising targets have the lowest relative uncertainty. By aligning the learning signal with this direction, the weighting scheme prioritizes more reliable components of the score field. \vspace{-6mm}}
    \label{fig:leading figure}
\end{figure}
In this work, we revisit DSM from a statistical perspective and show that this commonly used surrogate introduces a fundamental and previously underappreciated effect: DSM is inherently heteroscedastic. That is, the variance of its regression targets depends strongly on both the noise level and the local geometry of the data distribution. As a result, different regions of the input space contribute unequally to the learning signal. As shown in~\cref{fig:leading figure} for a mixture of 2D Gaussians, points in convex density regions (e.g., between modes) have highly variable denoising targets, while other points have more consistent targets. We prove that this phenomenon is not an artifact of the model or parameterizations, but an intrinsic consequence of replacing score matching with conditional (denoising) targets. Ignoring this heteroscedasticity leads to statistically inefficient estimators and can adversely affect optimization dynamics~\citep{gourieroux1984pseudo, hullermeier2021aleatoric}.

Our analysis shows that this effect is unavoidable: even when the model class is sufficiently expressive, the denoising objective contains an irreducible variance term that does not appear in score matching. This establishes a fundamental statistical gap between the two formulations. From this perspective, DSM should not be viewed as a "\emph{free}" approximation to score matching, but rather as a different estimation problem with its own noise structure. To understand how this issue can be addressed, we derive the variance-optimal (Godambe) weighting that equalizes the contribution of different samples. This weighting renders the objective homoscedastic in principle and recovers the best linear unbiased estimator in ideal settings. However, this optimal solution depends on conditional variance terms (equivalently, higher-order information such as local curvature) that are generally intractable to compute in high-dimensional settings. This reveals a fundamental tension: DSM can be made statistically efficient in theory, but doing so requires quantities that are impractical to estimate.

Furthermore, we develop a practical approximation of this weighting function based on a first-order Taylor expansion. We show that this computationally efficient approximation partially mitigates heteroscedasticity. Crucially, we demonstrate that the ubiquitous $\sigma_t^2$ heuristic widely used by isotropic Gaussian diffusion models~\citep{song2020score, song2021maximum} emerges as a specific case of our analysis,  providing a theoretical explanation for its empirical success, while clarifying its limitations.

Finally, this work develops a general framework for arbitrary order denoising score matching under general smooth noise distributions, extending previous results that rely on Gaussian or exponential family assumptions. This framework makes explicit how heteroscedasticity arises across different settings and provides a common lens for analyzing existing methods.
Our core contributions are:
\begin{itemize}[leftmargin=0.5cm]
    \item \textbf{Intrinsic heteroscedasticity of DSM.}
    We show that denoising score matching induces input-dependent variance that depends on both the noise level and the data geometry, revealing a fundamental statistical difference from score matching (\cref{thm:loss-decomposition}).

    \item \textbf{Optimal weighting and its limits.}
    We derive the variance-optimal (Godambe) weighting that restores homoscedasticity, and show that it depends on intractable quantities, making exact correction impractical (Lemma \ref{lemma: weighting function}).

     \item \textbf{Theoretical grounding for existing heuristics.}
    Our framework generalizes across arbitrary score orders and general smooth noise distributions. We show that widely used heuristic weights ($\sigma_t^2$ in Isotropic Gaussian diffusion) emerge as specific instances of the optimal solution (Lemma \ref{lemma: practical approximation of weighting function}).

    \item \textbf{A fundamental tradeoff.}
    We derive a practical approximation of the weighting function and show that statistically optimal weighting can increase gradient variance, while the approximation improves optimization at the cost of statistical efficiency (Lemma \ref{lem: sandwich form for variance}).
   
\end{itemize}

\section{PRELIMINARIES}

\paragraph{Heteroscedasticity and weighted least squares:} Heteroscedasticity refers to the problem where the uncertainty of a model's output depends on input features. Failing to account for this input-dependent variance forces the model to fit all regions of the input space uniformly, regardless of their uncertainty, which can slow convergence and encourage overfitting~\citep{nix1994estimating,hullermeier2021aleatoric}. In nonlinear models such as neural networks, these effects can be more severe, potentially leading to biased and inconsistent parameter estimates~\citep{gourieroux1984pseudo}. Let $x \in \mathcal{X}$ denote an input, and let the corresponding model output be subject to random noise $\epsilon$. In the homoscedastic setting, the conditional variance of the noise is constant, 
\begin{equation}
    \mathbb{V}(\epsilon \mid x) = \sigma^2,
\end{equation}
whereas under heteroscedasticity, this variance depends on the input, 
\begin{equation}
    \mathbb{V}(\epsilon \mid x) = \sigma^2(x).
\end{equation}
A principled way to address heteroscedasticity uses a weighting function that is inversely proportional to the variance of the uncertainty~\citep{kiers1997weighted}. This ensures that samples with more noise contribute less to the loss, effectively equalizing variance across inputs. This arises naturally in weighted least squares, yielding the best linear unbiased estimator (BLUE) for linear (kernel) regression.

\paragraph{Pythagorean Decomposition of the $L_2$ Form:}
A key property of the multivariate squared-error loss is that it admits a bias–variance split in $L_2$.  In particular, 
\begin{definition}[Pythagorean $L_2$–Decomposition]\label{def:bias-var}
Let $x$ and $y$ be random variables, and let 
$f:\mathcal X\to\mathbb R^d$ and $g:\mathcal X\times\mathcal Y\to\mathbb R^d$ be measurable functions.  Then for each $x\in\mathcal X$,
\begin{align}
  &\mathbb E_{y| x}\bigl[\|f(x)-g(x,y)\|_2^2\bigr] \notag \\
  &= \underbrace{\|f(x)-\mathbb E_{y| x}[g(x,y)]\|_2^2}_{\text{bias}} + \text{tr} \underbrace{\mathbb{V}_{y| x}\bigl(g(x,y)\bigr)}_{\text{variance}}
\end{align}
Here $\mathbb E_{y| x}$ and $\mathbb V_{y| x}$ denote the conditional expectation and variance given $x$.
\end{definition}

\paragraph{First-order Score Matching and Denoising Score Matching:}
For first-order score estimation, score matching learns a parametric approximation
$s_\theta(x,t)$ of the score $\nabla_x \log p_t(x)$ by minimizing
\begin{equation}
\mathcal{L}_{\mathrm{SM}}(\theta)
=
\mathbb{E}_{p_t(x)}
\left[
\left\|
\nabla_x \log p_t(x) - s_\theta(x,t)
\right\|_2^2
\right].
\end{equation}
Since the score of the perturbed density $p_t(x)$ is generally intractable, denoising score matching instead replaces it with the conditional score $\nabla_{x_t}\log p(x_t\mid x_0)$ and optimizes
\begin{equation}
\mathcal{L}_{\mathrm{DSM}}(\theta)
=
\mathbb{E}_{p(x_0,x_t)}
\left[
\left\|
\nabla_{x_t}\log p(x_t\mid x_0) - s_\theta(x_t,t)
\right\|_2^2
\right].
\end{equation}

\cite{vincent2011connection} shows the link between first-order score matching and denoising score matching:
\begin{equation}
    \underbrace{\E_{p_t(x)}[\left\lvert\lvert \nabla_{x}\log p_t(x) - s_\theta(x, t) \right\rvert\rvert^2_2}_{\text{Score Matching}} 
    = \underbrace{\E_{p(x_0, x_t)}[ || \nabla_{x_t}\log p(x_t|x_0) - s_\theta(x_t, t)||^2_2]}_{\text{Denoising Score Matching}} + \text{Difference} \label{eq: first order dsm relationship}
\end{equation}

Because the "Difference" term is strictly independent of the model parameters $\theta$, it is concluded that the two objectives are equivalent in expectation.

\paragraph{Weighting in Denoising Score Matching:}
A particularly important instance of a weighting function arises in isotropic Gaussian diffusion models~\citep{song2021maximum}, which rely on denoising score matching. Here, when training using noise perturbation, the noise level is explicitly controlled by a variance schedule $\sigma^2(t)$ at each diffusion step $t$. To account for this, \cite{song2021maximum} applies a weighting function $w(t) \propto \sigma^2(t)$, motivated by an empirical analysis of the magnitude of score function.
This aims to balance the score magnitude across different noise levels~\citep{song2020score}. However, as shown below, this weighting scheme does not render the loss function homoscedastic.

\section{ANALYSIS}
To understand the heteroscedasticity of DSM, we first derive the connection between arbitrary order SM and DSM (\cref{sub_sec: DSM framework}). Building on this foundation, we demonstrate that this relationship is inherently heteroscedastic (\cref{sub_sec: Hetero in DSM}), which motivates the derivation of a Godambe-optimal weighting scheme applicable to arbitrary orders and smooth noise distributions. Finally, we introduce a practical approximation of this weighting function across various distributions and examine its relationship to commonly used heuristic weighting schemes (\cref{sub_sec: weighting function}). The corresponding proofs for all theorems and lemmas can be found in Appendix \ref{sec: proof}.

\subsection{Arbitrary order DSM with generalized Noise}
\label{sub_sec: DSM framework}

We now present an iterative recursion for arbitrary–order score functions (\cref{thm:iterative-score}), which is used to show heteroscedasticity is a structural property of DSM that holds beyond first-order settings.

We start with the first and second order functions introduced in~\cite{lu2022maximum}. Although introduced in the context of a Gaussian perturbed kernel, this derivation is also applicable to other distributions.

\begin{definition}\label{def:first-order-score}
A Monte Carlo estimator of the first-order score function, $s_1(x_t)$, is 
$$
  s_1(x_t)
  = \E_{x_0| x_t}\bigl[s_1(x_t| x_0)\bigr],
$$
where $s_1(x_t| x_0) = \nabla_{x_t}\log p(x_t| x_0)$.

Differentiating both sides with respect to $x_t$ yields the second-order score function $s_2(x_t)$
\begin{align}
    s_2(x_t) &= \nabla_{x_t}\E_{x_0| x_t}\bigl[s_1(x_t| x_0)\bigr] \notag \\
    &= \E_{x_0| x_t}\Bigl[s_1(x_t| x_0)s_1(x_t| x_0)^\top
    -s_1(x_t| x_0)s_1(x_t)^\top + s_2(x_t| x_0)\Bigr]
\end{align}
where the conditional second‐order score is
$$
  s_2(x_t| x_0)
  =
  \nabla_{x_t}^2\log p(x_t| x_0).
$$
\end{definition}

\begin{remark}
    The estimator for $s_2(x_t)$ is found by taking derivatives on both sides of the Monte Carlo estimator of the first-order score function. Then $\nabla_{x_t} \log p(x_0|x_t)$ can be further expanded using Bayes' rule at the score level to obtain the final equation. The proof can be found in \cref{prf: first-order-score}.
\end{remark}

Given the first-order score identity, we can substitute this into the second-order identity, which yields multiple variants of the Monte Carlo estimate of the second-order score function. A brief analysis of these estimators is provided in \cref{sec: analysis of second order estimator} in terms of their bias and variances.

We now introduce the iterative form for an arbitrary order score function.

\begin{theorem}\label{thm:iterative-score}
Given the data distribution $p(x_0)$ and the noising distribution  $p(x_t| x_0)$, the $k$th‐order score is defined as 
$$
  s_k(x_t)=\nabla_{x_t}^k\log p(x_t).
$$
Initializing
$$
  h_1(x_0,x_t)=s_1(x_t| x_0)
$$
Then for all $k\ge2$:
$$
  s_k(x_t)
  =
  \E_{x_0| x_t}\bigl[h_k(x_0,x_t)\bigr]
$$
where
\begin{align*}
  &h_{k}(x_0,x_t) = \nabla_{x_t} h_{k-1}(x_0, x_t) + h_{k-1}(x_0, x_t) \otimes \bigl(s_1(x_t|x_0) - s_1(x_t) \bigr)
\end{align*}
\end{theorem}

\begin{remark}
In the first‐order Gaussian case, 
$
  h_1(x_0,x_t)=s_1(x_t| x_0),
$
which is simply the score of the perturbed kernel with Gaussian perturbations $x_t| x_0\sim\mathcal N(x_0,\sigma_t)$).  More generally, at order $k=K$ the function
$
  h_K(x_0,x_t)
$ collects the conditional score $s_K(x_t| x_0)$, which remains known analytically, and the previously computed scores $s_1(x_t),\dots,s_{K-1}(x_t)$.  Solving the recursion in Theorem \ref{thm:iterative-score} then yields the new marginal score $s_K(x_t)$. The proof can be found in~\cref{Prf: thm:iterative-score}
\end{remark}

\citet{meng2021estimating} use Tweedie’s formula to derive arbitrary‐order denoising score matching. However, Tweedie's formula is restricted to perturbed kernels $p(x_t| x_0)$ in the exponential family. \citet{lu2022maximum} provides explicit formulas for second and third‐order scores, but these are discussed in the context of Gaussian noise. Similarly, \citet{manor2023posterior} provide an iterative formula to derive higher-order central moments under the assumption of the noise distribution being Gaussian. The theorems above provide a general formula without exponential family noise distribution assumptions.

To summarize, Theorem \ref{thm:iterative-score} above provides a general and iterative framework for obtaining the $k$th‐order score function with smoothness assumptions on $p(x_t| x_0)$. 

\subsubsection{Loss Function for arbitrary order Denoising Score Matching}

We now show that the score-matching loss admits a variance-based decomposition, which clarifies its relationship to the denoising score-matching loss. To make this precise, we analyze how replacing the true score with conditional targets affects the loss. In particular, we show that this substitution introduces an additional variance term that does not appear in standard score matching.

\begin{theorem}\label{thm:loss-decomposition}
Assume that for each $x_t \in \mathcal{R}^{d}$ where d is the dimensionality of the samples, the function $h_k(x_0,x_t)$ satisfies
$$
s_k(x_t)=\E_{x_0| x_t}\bigl[h_k(x_0,x_t)\bigr],
$$
Let us define the weighted score matching loss $\mathcal{L}_{\text{sm}}$ and the weighted denoising score matching loss $\mathcal{L}_{\text{dsm}}$ using an $L_2$ norm\footnote{Pointwise convergence with $L_2$ norm~\citep{ziemer2017modern}}:
\begin{align}
 \mathcal{L}_{sm}  
&=\E_{p(x_t)}\Big(\big(s_k(x_t;\theta)-s_k(x_t)\big)^\top W \big(s_k(x_t;\theta)-s_k(x_t)\big)\Big) \label{eq: score matching loss} \\
 \mathcal{L}_{dsm}  
&=\E_{p(x_t,x_0)}\Big(\big(s_k(x_t;\theta)-h_k(x_0,x_t)\big)^\top W 
\big(s_k(x_t;\theta)-h_k(x_0,x_t)\big)\Big) \label{eq: denoising score matching loss}
\end{align}
Using definition \ref{def:bias-var}, these are connected as follows
$$
\mathcal{L}_{sm}
=\mathcal{L}_{dsm}
-\tr\E_{p(x_t)}\bigl[W\mathbb{V}_{x_0| x_t}\bigl(h_k(x_0,x_t)\bigr)\bigr].
$$
\end{theorem}

\begin{remark}
When the model $s_k(x_t;\theta)$ has enough capacity and converges perfectly, the primary score‐matching loss $\mathcal L_{sm}$ goes to zero. However, the denoising variant $\mathcal L_{dsm}$ cannot go all the way to zero; its best‐case value is
$ \tr\E_{p(x_t)}\bigl[WV_{x_0| x_t}\bigl(h_k(x_0,x_t)\bigr)\bigr]$, 
reflecting the irreducible conditional‐variance term if $W$ is constant with respect to $x_t$ and $t$. This term represents a fundamental limitation: even with a perfect model, the DSM objective cannot eliminate this variance. The detailed proof can be found in~\cref{prf: loss-decomposition}

When $k=1$, we recover the relationship established by \citet{vincent2011connection}, as shown in \cref{eq: first order dsm relationship}.
\end{remark}

\subsection{Heteroscedasticity in Denoising Score Matching}
\label{sub_sec: Hetero in DSM}


Theorem~\ref{thm:loss-decomposition} may appear unintuitive, since its final irreducible term is strictly independent of the model parameters. However, this term is precisely where heteroscedasticity manifests in DSM: it reveals that the variance of the loss fluctuates dynamically based on both the input noise level and the local data geometry. To clarify its practical impact, we further analyze how this input-dependent variance propagates into both the parameter gradients and the model parameters themselves, which eventually influences the learning dynamics.

\begin{definition}
    Given the denoising score loss defined in \cref{eq: denoising score matching loss} and the parametrized model $s_k(x_t; \theta)$, the corresponding gradient $g$ of the loss function is defined as 
    \begin{align}
        g_{\text{dsm}}(\theta, \sigma_t)
        &= \partial_\theta \mathcal{L}_{\text{dsm}} \notag \\
        &= 2\E_{x_0, x_t} \Big( J(x_t; \theta)^\top W \big(s_k(x_t; \theta) - h_k(x_0, x_t)\big)\Big) \label{eq: gradient function of dsm}
    \end{align}
    where $J(\theta)$ is the Jacobian matrix of the score function with respect to $\theta$, i.e. $J(\theta) = \partial_\theta s(x_t; \theta)$
\end{definition}



\begin{lemma}\label{lem: sandwich form for variance}
    Given the model has enough capacity and the optimal parameters $\theta^{*}$, the variance of the model parameters $\theta$ and the variance of gradient of parameters $g_{dsm}$ are 
    \begin{align}
        \begin{cases}
            \mathbb{V}(\theta)
            &= \E_{x_t} \Big(S(x_t, \theta^{*}) \notag G(x_t, t)S(x_t, \theta^{*})^\top\Big) \\
            \mathbb{V}(g_{\text{dsm}}) &=4 \E_{x_t} \left(J(x_t; \theta)^\top G(x_t, t)J(x_t; \theta) \right) 
        \end{cases} \label{eq: variances}
    \end{align}
    where 
    \begin{align}
        \begin{cases}
        &H_{\mathcal{L}}( x_t, W ;\theta^{*}) = 2J^\top(x_t; \theta^{*}) W J(x_t; \theta^{*}) \\
        &S(x_t, \theta^{*}) = H_{\mathcal{L}}( x_t, W ;\theta^{*})^{-1}J(x_t; \theta^{*})^\top  \\
        &G(x_t, t) = \E_{x_0|x_t}\bigg(W V_{x_0|x_t}(h_k(x_t, x_0))W^\top)\bigg) 
    \end{cases}
    \end{align}
\end{lemma}

\begin{remark}
    Equation \ref{eq: variances} explicitly isolates the source of this heteroscedasticity. At convergence, the model-dependent term $S(x_t, \theta^{*})$ is fixed. The heteroscedasticity is instead driven entirely by the matrix $G$, specifically through the irreducible variance term $V_{x_0|x_t}(h_k)$. Because this term depends dynamically on the input $x_t$ and the noise level $t$ and is strictly independent of the model parameters $\theta$, it acts as an unavoidable, underlying noise floor. When this variance is large, it directly destabilizes both the parameter gradients and the model parameters themselves. Consequently, the irreducible variance in $G$ represents the fundamental "statistical cost" of using Denoising Score Matching as a tractable surrogate for true Score Matching. The detailed proof can be found in~\cref{prf: sandwich form for variance}
\end{remark}


%




\subsection{Godambe-optimal Weighting Function for arbitrary order DSM and Its approximation}
\label{sub_sec: weighting function}

Having identified the source of heteroscedasticity, where different noise levels $t$ and $x_t$ contribute unequally to the DSM loss, the Godambe framework~\cite{godambe1991estimating} provides the ideal weights that cancel out this effect: it rescales the 
conditional variance of the noisy score estimate $h_k(x_0,x_t)$ so that, on average, each $x_t$ contributes equally. 

\begin{lemma} \label{lemma: weighting function}
    Formally, the Godambe-optimal weight matrix $W$ is chosen to satisfy
    \begin{equation}
        W \mathbb{V}_{x_0|x_t}\left[h_k(x_0, x_t)\right] = I
    \end{equation}
    \begin{equation}
    \therefore W = \mathbb{V}_{x_0|x_t}\left[h_k(x_0, x_t)\right]^{-1}
    \end{equation} 

\end{lemma}

\begin{remark}
Lemma \ref{lemma: weighting function} gives the Godambe-optimal weights for arbitrary denoising score matching with generalized smooth noise. When the model is parameterized as a linear regression of different features of $(x_t, t)$, it can be shown that the ideal weight matrix leads to the Best Linear Unbiased Estimator (BLUE), where the best refers to the estimator with minimum parameter variance. We refer the reader to Appendix \ref{prf: hetero analysis} for the proof.
\end{remark}


This gives the ideal weighting function for arbitrary order score matching for any smooth noise distribution when using denoising score matching as a proxy for score matching.

\paragraph{Approximation of the Statistical Optimal Weighting Function}
\begin{lemma}[Approximation of Optimal Weight Function] \label{lemma: practical approximation of weighting function}
    For arbitrary denoising score matching, the expected Godambe-optimal weight matrix $W(t)$ can be reduced to the following form with first-order Taylor expansion:
    \begin{align}
        W &= \E_{x_t} \left[ \mathbb{V}_{x_0|x_t}\left[h_k(x_0, x_t)\right] \right]^{-1} \\
        &= \left(\E_{x_0} \fisher_{x_t|x_0}^{(k)} \right)^{-1} + \mathcal{O}\left( \big\| (\fisher_{x_t|x_0}^{(k)})^{-1} \big\|^2 \right) \\
        &\approx \left(\E_{x_0} \fisher_{x_t|x_0}^{(k)} \right)^{-1} \label{eq: approxmiation weighting function}
    \end{align}
    where $\fisher^{(k)}_{x_t|x_0}$ is the $k$th order Fisher information tensor of the perturbation kernel $p(x_t|x_0)$.
\end{lemma}

\begin{remark}
    In practice, most Gaussian DSM choose a noising distribution whose Fisher information does not depend on the input data $x_0$. For example, under isotropic Gaussian noise where $p(x_t|x_0) = \mathcal{N}(x_t; x_0, \sigma_t^2 I)$, the $k$th order Fisher information simplifies to  $\fisher^{(k)}_{x_t|x_0} = \sigma_t^{-2k} I$. The detailed proof can be found in~\cref{prf: practical approx}
\end{remark}

Interestingly, this approximation can achieve lower variance of model parameter gradients $g_{\text{dsm}}$ than the statistical optimal weighting despite achieving higher variance of model parameters $\theta$. 

\begin{lemma}[Variance with weighting functions] \label{variance properties of approximated weights}

 The \emph{\textbf{variance of parameters}} using the ideal weight matrix $W = \mathbb{V}_{x_0|x_t}(h_k(x_0, x_t))^{-1}$ is strictly less than or equal to that of the approximation
    $$\mathbb{V}(\theta)_{\text{ideal}} \leq \mathbb{V}(\theta)_{\text{approx}}$$

    However, the \emph{\textbf{variance of the gradient}} of the parameters using the ideal weights is larger than that obtained when using the approximation in expectation.
    $$\E\big(\mathbb{V}(g_{dsm})_{\text{ideal}}\big)  > \E\big(\mathbb{V}(g_{dsm})_{\text{approx}}\big)$$
\end{lemma}

The detailed proof can be found in~\cref{prf: variance properties of approx}

\section{EVALUATIONS}

\paragraph{Beyond First-order Scores:} We evaluate the convergence behavior of second-order score estimation on a Gaussian Mixture Model (GMM) where computation of the ideal weight is tractable. As shown in \cref{fig: 2nd experiment}, the unweighted baseline (No Weight) suffers from severe instability and poor convergence across most noise levels. By contrast, incorporating weighting schemes tailored for heteroscedasticity drastically reduces gradient variance and accelerates convergence. Furthermore, while the Taylor-approximated weights successfully stabilize early training, they exhibit late-stage divergence in low-noise regimes. The ideal weighting scheme completely prevents this divergence.

\begin{figure}[htb]
    \centering
    \includegraphics[width=0.95\linewidth, trim=0cm 10cm 0cm 0cm, clip]{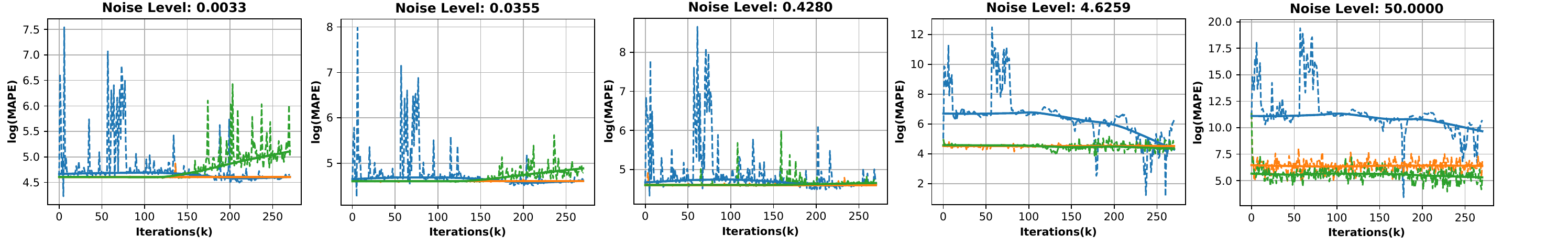}
    \vspace{2pt} 
    {\footnotesize \color{pastelblue}--- \textbf{No Weight}} \quad {\footnotesize \color{pastelgreen}--- \textbf{Approximation}} \quad {\footnotesize \color{pastelred}--- \textbf{Ideal Weight}}
    
    \caption{\textbf{Convergence of second-order DSM on a GMM for varying noise levels.} The plots compare the $\log(\text{MAPE})$ over training iterations for the unweighted baseline,  proposed Taylor approximation, and the ideal weighting. While the unweighted objective exhibits extreme variance, both weighting strategies successfully stabilize training. Notably, the ideal weight bounds the error, preventing the late-stage divergence observed with the approximated weights at lower noise levels.}
    \label{fig: 2nd experiment}
    \vspace{-5mm}
\end{figure}

\paragraph{Beyond Gaussian Diffusion:}
To evaluate the training quality and verify that the proposed approximate weighting scheme reduces heteroscedasticity, we measure the Mean Absolute Percentage Error (MAPE) against the exact, analytical true score. Because deriving this ground-truth score is mathematically intractable for complex, high-dimensional data, our empirical validation is necessarily restricted to simpler distributions where the true score is known in closed form. We evaluate how our approximate weight (Lemma \ref{lemma: practical approximation of weighting function}) improves performance on a 1D Mixture of Cauchy dataset corrupted by various noise families. \cref{tab:cauchy_experiment} shows that naive training without weights results in severe performance degradation, particularly at extreme noise levels. In general, applying the derived first-order optimal weights substantially reduces the MAPE of the predicted score.

\begin{table}[!htb]
\centering
\caption{Performance (MAPE) of predicted score vs. true score on the 1D Mixture of Cauchy dataset noised with different noise distributions\protect\footnotemark. Lower values are better.}
\vspace{5pt}
\label{tab:cauchy_experiment}
\resizebox{0.8\columnwidth}{!}{%
\begin{tabular}{lccccc}
    \toprule
    \textbf{Noise Distribution} & \textbf{Weights} & 0.12 & 1.4 & 15 & 50 \\
    \midrule
    Cauchy$(0, \gamma_t)$ & No Weights & \textbf{1.20$\pm$0.10} & 1.40$\pm$0.29 & 7.84$\pm$4.96 & 19.95$\pm$15.36 \\
     & 1st-Order Approx. & 3.83$\pm$1.26 & \textbf{0.48$\pm$0.21} & \textbf{0.30$\pm$0.06} & \textbf{0.72$\pm$0.25} \\
    \midrule
    Gen.\ Normal$(0, \alpha_t, \beta_t)$ & No Weights & 1.96$\pm$0.62 & 3.93$\pm$2.52 & 15.32$\pm$10.91 & 7.75$\pm$1.28 \\
     & 1st-Order Approx. & \textbf{1.91$\pm$0.53} & \textbf{0.31$\pm$0.09} & \textbf{0.19$\pm$0.09} & \textbf{0.35$\pm$0.18} \\
    \midrule
    Laplace$(0, \gamma_t)$ & No Weights & \textbf{1.25$\pm$0.31} & 3.42$\pm$1.55 & 6.37$\pm$3.77 & 8.93$\pm$1.56 \\
     & 1st-Order Approx. & 1.27$\pm$0.06 & \textbf{0.51$\pm$0.14} & \textbf{0.14$\pm$0.02} & \textbf{0.27$\pm$0.02} \\
    \midrule
    \text{Student} $t_3(0, \gamma_t)$ & No Weights & 2.20$\pm$1.07 & 4.82$\pm$2.48 & 17.64$\pm$12.28 & 74.75$\pm$35.47 \\
     & 1st-Order Approx. & \textbf{1.40$\pm$0.44} & \textbf{0.35$\pm$0.19} & \textbf{0.21$\pm$0.07} & \textbf{0.38$\pm$0.16} \\
    \bottomrule
\end{tabular}
}
\end{table}

\footnotetext{The approximated weighting functions are detailed in \cref{app: cachy_exp_settings}}





\paragraph{Alternative Weighting Function Approximations:}
The first order weighting function approximation above disregards curvature information necessary for statistical efficiency, ignoring the underlying geometry of the data. 
Unfortunately, computing the Hessian for high-dimensional data is prohibitively expensive, scaling as $\mathcal{O}(D^2)$. One way to circumvent this computational bottleneck is computing the trace of Hessian instead of the full Hessian. Below, we approximate the trace of the Hessian of the marginal log-density using the divergence of the learned score model, estimated with Hutchinson probes \citep{hutchinson1989stochastic}. We tested this second-order approximation on standard generative datasets. As demonstrated in \cref{tab:hutchinson_experiment}, incorporating these scalar curvature approximations provides minor improvements in  distribution modeling over the first-order approximation.
\begin{table}
    \centering
    \caption{Learned score qualities (MAPE of predicted score vs target, lower is better) on various datasets for a Gaussian diffusion application. }
    \vspace{5pt}
    \resizebox{0.95\columnwidth}{!}{\begin{tabular}{lcccccc}
        \toprule
        \textbf{Weighting} & \textbf{Spiral} & \textbf{Moons} & \textbf{GMM} & \textbf{Swissroll} & \textbf{MNIST} & \textbf{CIFAR-10}  \\
        \midrule
        No Weighting & $4.02 \pm 0.06$ & $4.04 \pm 0.05$ & $3.94 \pm 0.05$ & $3.99 \pm 0.04$ & $4.10 \pm 0.0001$ & $4.10 \pm 0.0001$\\
        First-Order Approximation & $3.97 \pm 0.01$ & $3.95 \pm 0.02$ & $3.93 \pm 0.02$ & $3.95 \pm 0.01$ & $4.00 \pm 0.0002$ & $4.10 \pm 0.0001$\\
        Hutchinson Estimator & $\mathbf{3.96 \pm 0.05}$ & $\mathbf{3.94 \pm 0.01}$ & $\mathbf{3.91 \pm 0.01}$ & $\mathbf{3.94 \pm 0.01}$ & $\mathbf{3.99 \pm 0.0003}$ & $\mathbf{4.10 \pm 0.000}$\\
        \bottomrule
    \end{tabular}%
    }
    \label{tab:hutchinson_experiment}
\end{table}
This reflects an inherent difficulty: accurately estimating curvature in high dimensions is computationally expensive, and low-rank or scalar approximations fail to capture the full variance structure. 

\paragraph{Empirical Analysis of Model Parameter Variances:}
\cref{fig:variance of parameter gradients dim 10} shows that the variances of the gradients of model parameters for ideal and first-order approximate weighting functions are similar when $\sigma_t$ is low, but heuristic weighting achieves significantly lower variance than the ideal weights.

\begin{figure}[!htb]
    \centering
    \vspace{-4mm}
    \resizebox{0.8\columnwidth}{!}{
        \input{figs/model_size_10}
    }
    \vspace{-4mm}
\caption{Variance of gradients of model parameters: As $\sigma_t$ becomes larger, approximated weights have much lower variance than the Godambe Ideal Weights. $std$ denotes the standard deviation.}
    \label{fig:variance of parameter gradients dim 10}
     \vspace{-6mm}
\end{figure}

\section{DISCUSSION AND LIMITATIONS}

\paragraph{Statistical vs Optimization tradeoff} Our analysis shows that denoising score matching is not merely a noisy approximation of score matching, but a fundamentally different estimation problem with input-dependent variance. While a variance-optimal correction exists in principle, it depends on intractable quantities. Practical training objectives necessarily operate under a tradeoff between statistical efficiency and optimization stability.
The variance optimal weighting minimises the asymptotic variance of the estimator, but can increase the variance of stochastic gradients during training. In contrast, heuristic weightings reduce gradient variance and improve optimization stability, but yield statistically suboptimal estimators. This tradeoff suggests that improving DSM requires balancing these two objectives, rather than optimizing either in isolation. In practice, this means that weights that emphasize statistically reliable regions can amplify noisy gradients during training, while weights that stabilize optimization may underutilize informative samples.

\paragraph{Implications for Gaussian diffusion}
In standard score-based diffusion models, the noise distribution is typically isotropic Gaussian. Under this assumption, our analysis shows that the commonly used $\sigma_t^2$ heuristic weighting aligns directly with the proposed approximation. We believe that this explains why this simple heuristic weighting scheme performs well in practice: it implicitly captures the dominant component of the variance. However, our analysis also clarifies the scheme's limitations. 
It does not fully correct the heteroscedasticity of DSM, and its effectiveness depends on the structure of the data and noise, e.g. particularly in distributions with high local concavity, where a scalar weighting fails to capture the complex local geometry of the score field.

\paragraph{Beyond first-order Gaussian settings}
Higher-order scores capture uncertainty and curvature information that first-order scores cannot~\citep{meng2021estimating}, leading to improved likelihood estimation through better-conditioned objectives~\citep{lu2022maximum}. Moreover, higher-order scores have the potential to significantly accelerate sampling algorithms in diffusion models~\citep{dockhorn2022genie, meng2021estimating}. As research in this specific domain remains largely theoretical, the primary utility of our unified weighting approximation is to stabilize the training of these higher-order models to allow practical integration into advanced sampling algorithms~\citep{dockhorn2022genie} in future work.

Non-Gaussian diffusions often outperform standard Gaussian models on complex topologies like long-tailed data~\citep{deasy2021heavy, pandey2024heavy}. Our framework provides a principled method for deriving weighting functions for these models, filling an important gap where heuristic strategies are currently absent. However, because non-Gaussian noise can introduce discontinuities in SDE sampling algorithms, translating these stabilized training objectives into efficient generation will require the development of specialized, robust samplers in future work.

\section{RELATED WORK}
The standard objective used to train continuous-time diffusion models enforces consistency between the marginal distributions parameterized by the neural network and those of the analytical reverse chain~\citep{song2021maximum}. This yields an idealized score matching loss that includes a weighting function designed specifically to ensure this marginal consistency, a role extensively explored through the lens of the ELBO and likelihood optimization~\citep{kingma2021variational, kingma2023understanding, song2021maximum}. 
 Since the true score is generally inaccessible, optimization relies on Denoising Score Matching (DSM) as a tractable surrogate. As we have shown, using this proxy intrinsically alters the optimization landscape by introducing input-dependent estimation variance (heteroscedasticity). Consequently, the practical objective requires a  \textit{different} weighting factor that additionally corrects for the proxy variance. 
Prior timestep-dependent weighting schedules have focused almost exclusively on matching the diffusion trajectory, relying on heuristic designs~\citep{kumar2025loss, gnanasambandam2020one} or empirical tuning tailored to specific architectures~\citep{Hang_2023_ICCV, karras2022elucidating}. Standard heuristics (eg. rescaling samples across noise levels in isotropic Gaussian diffusion~\citep{song2021maximum}) act only as practical compensations for this uncharacterized variance. While recent work by \citet{xu2023stable} proposes self-normalized importance sampling to reduce target variance, the distinct statistical weighting requirements for variance-correction remain unresolved. To the best of our knowledge, no prior work explicitly isolates this objective by framing DSM as a heteroscedastic estimation problem. 

\begin{table}[htb]
    \centering
    \vspace{-4mm}
    \caption{Summary of our framework's coverage compared to existing heuristics}
    \vspace{5pt}
    \resizebox{0.95\columnwidth}{!}{%
    \begin{tabular}{lccc}
        \toprule
         \textbf{Weighting function for DSM} & \textbf{Smooth Noise Distribution} & \textbf{Prior heuristics proposed} & \textbf{Our Analysis (Optimal \& Approx)} \\
         \midrule
         \multirow{4}{*}{\textbf{First-Order DSM}} 
         & Isotropic Gaussian & \cite{song2020score} & \cmark \\
         \addlinespace
         & Other Exponential & Generalized Normal~\citep{deasy2021heavy} & \cmark \\
         \addlinespace
         & Non-Exponential & t Distribution~\citep{pandey2024heavy} &  \cmark \\
         \midrule
         \multirow{4}{*}{\textbf{Higher-Order DSM}} 
         & Isotropic Gaussian & \citep{lu2022maximum, meng2021estimating} & \cmark \\
         \addlinespace
         & Other Exponential & \xmark & \cmark \\
         \addlinespace
         & Non-Exponential & \xmark & \cmark \\
         \bottomrule
    \end{tabular}%
    }
    \label{tab:Implications}
    \vspace{-2mm}
\end{table}

\section{CONCLUSION}
This work began with a fundamental question: Is Denoising Score Matching (DSM) truly "Score Matching for free"? Our findings mathematically demonstrate that it is not. We established that DSM is inherently heteroscedastic, revealing a fundamental statistical tension between the convenience of the denoising proxy and the stability of true Score Matching. To address this limitation, we developed a framework to derive the Godambe-optimal weighting scheme for arbitrary order scores and general smooth noise distributions. As summarized in Table \ref{tab:Implications} our analysis formalises existing literature and extends it more generally to other settings. 
This work presents several avenues for future research. To circumvent the prohibitive cost of computing Hessians, our approximation relies entirely on the Fisher information of the noise distribution, so residual heteroscedasticity may persist when modeling highly complex target distributions. Second, while our findings for standard Gaussian settings are implicitly validated at scale by existing state-of-the-art models, large-scale empirical validation for general non-exponential or heavy-tailed distributions remains limited. Applying our framework to these novel domains at scale will require the development of specialized sampling algorithms capable of handling potential discontinuities in the corresponding Stochastic Differential Equations (SDEs).


\newpage
\bibliographystyle{apalike}
\bibliography{ref}


\appendix
\thispagestyle{empty}

\onecolumn

\section{Evaluation Details}
All experiments were conducted on a single NVIDIA GeForce RTX 2080 Ti GPU. Below, we provide detailed descriptions of the evaluations corresponding to each figure.

Noting that the maximum value of the Godambe weighting function is too large to train the model numerically, we clamp the maximum weighting value by the maximum of its first order approximation $\sigma_t^2$ whenever $W_{\text{godambe}} / \sigma_t^2 > 5$ for the scalar case. Generally for $x_t \in \mathcal{R}^d$, we use the weight as
\begin{align}
    \widetilde W(x_t,\sigma_t)
&=
\sigma_t^2\left(I+\sigma_t^2 H(x_t)\right)^{-1}.
\\
W(x_t,\sigma_t)
&=
\begin{cases}
\sigma_t^2 I, 
& \text{if } \lambda_{\max}\!\left(\widetilde W(x_t,\sigma_t)\right)/\sigma_t^2 > 5, \\[4pt]
\widetilde W(x_t,\sigma_t),
& \text{otherwise.}
\end{cases}
\end{align}

\subsection{Implementation Details for Figures \& Experiments}

\subsubsection{\autoref{fig:leading figure}}
\autoref{fig:leading figure} is drawn based on a 2D Gaussian mixture with seed 0 which results in:
\begin{align*}
    \begin{cases}
        \lambda = \begin{bmatrix}
            0.4956 & 0.5044
        \end{bmatrix}\\
        \mu_1 = \begin{bmatrix}
        5.7824 &  2.7868
        \end{bmatrix} \\
        \mu_2 =   \begin{bmatrix}
        -0.2920 &  4.1942 
        \end{bmatrix}      \\
        \Sigma_1^{-1} = \begin{bmatrix}
            0.5705 & -0.0213 \\
            -0.0213 &  0.5172
        \end{bmatrix} \\
        \Sigma_2^{-1} = \begin{bmatrix}
            0.5914 &  0.0620 \\
            0.0620 &  0.8990
        \end{bmatrix}
    \end{cases}
\end{align*}

We provide the implementation in the supplementary code.

\subsubsection{\autoref{fig: 2nd experiment}}

The data distribution is defined as:
\begin{align*}
    p(x_0) = 0.3258 \times \mathcal{N}(0, 0.5063^2) + 0.3316 \times \mathcal{N}(2.0, 0.7782^2) + 0.3426 \times \mathcal{N}(4, 0.0985^2).
\end{align*}

The resulting marginal density \(p(x_t)\), obtained by convolving \(p(x_0)\) with the Gaussian kernel, is given by:
{\footnotesize\begin{align*}
    p(x_t) = 0.3258 \times \mathcal{N}(0, 0.5063^2 + \sigma_t^2) + 0.3316 \times \mathcal{N}(2.0, 0.7782^2 + \sigma_t^2) + 0.3426 \times \mathcal{N}(4, 0.0985^2 + \sigma_t^2).
\end{align*}}

We derive the analytical second-order score for the Gaussian Mixture when perturbed with $\sigma_t$. For the approximated weighting function, the 2nd-order fisher information is $\sigma_t^4$.

For the second-order DSM, the Godambe weighting results in more stable convergence behaviour for the entire time. For heuristic weighting, the loss diverges after 150k iterations. In addition, without using weight, loss converges slowly.

For the second-order DSM, the Godambe weighting function is
\begin{align}
    W &= \mathbb{V}_{x_0|x_t} \Big( h_2(x_0, x_t)\Big)^{-1} \\
     &= \mathbb{V}_{x_0|x_t} \Big( s(x_t|x_0)s(x_t|x_0)^\top - s(x_t|x_0)s(x_t)^\top - \frac{I}{\sigma_t^2} \Big)^{-1}
\end{align}

If we define $\delta = s(x_t|x_0) - s(x_t)$, we have
\begin{align}
    \mathbb{V}_{x_0|x_t} \Big( h_2(x_0, x_t)\Big) &= \mathbb{V}_{x_0|x_t} \Big( (s(x_t) + \delta) \delta  \Big) \\
    &= s(x_t)^2 \mathbb{V}_{x_0|x_t} (\delta) + \mathbb{V}_{x_0|x_t}(\delta^2) + 2s(x_t)\mathbb{V}_{x_0|x_t}(\delta^3)
\end{align}

For a source distribution as a Gaussian mixture, the posterior is also a Gaussian mixture. Therefore
\begin{align}
    \mathbb{V}_{x_0|x_t} (\delta) &= \mathbb{V}(\frac{x_0 - x_t}{\sigma_t^2}) \\\
    &= \frac{1}{\sigma_t^4} \mathbb{V}_{x_0|x_t}(x_0) \\
    \mathbb{V}_{x_0|x_t}(\delta^2) &= 
     \frac{1}{\sigma_t^8}\mathbb{V}(x_0 - x_t)^2 \\
     &= \frac{1}{\sigma_t^8}\mathbb{V}_{x_0|x_t} \Big( x_0^2 + x_t^2 - 2x_t x_0 \Big)
\end{align}
The above formulation suggests that $\mathbb{V}_{x_0|x_t} \Big( h_2(x_0, x_t)\Big)$ can be obtained by calculating the posterior moment with the posterior distribution as a Gaussian Mixture. We provide the implementation in the supplementary code.

\subsubsection{\cref{tab:cauchy_experiment}}
\label{app: cachy_exp_settings}
The Mixture of Cauchy dataset is trained with 5 seeds (0-5). The mixture weights is drawn from uniform distribution and softmax between 0 and 1. The scale parameter is sample from uniform. And location parameter is fixed as (0, 1, 2).

The noise distributions are Cauchy(0, $\gamma_t$), Generalized Normal(0, $\alpha_t, \beta_t$), Laplace (0, $\gamma_t$) and Student t(0, $\gamma_t$) with degree of freedom $\nu=3$. 

The corresponding weighting functions for each noise distribution are
\begin{table}[!htb]
\centering
\caption{Approximated Weighting function for each noise distribution}
\vspace{0.3cm}
\begin{tabular}{lc}
    \toprule
    \textbf{Noise Distribution} & \textbf{Approx Weights} \\
    \midrule
    Cauchy$(0, \gamma_t)$ & $\gamma_t^2$ \\
    \midrule
    Gen.\ Normal$(0, \alpha_t, \beta_t)$ & $\frac{\alpha_t^2}{\beta_t}$ \\
    \midrule
    Laplace$(0, \gamma_t)$ & $\gamma_t^2$ \\
    \midrule
    \text{Student} $t_\nu(0, \gamma_t)$ & $\frac{\gamma_t^2 (\nu + 3)}{2\nu}$ \\
    \bottomrule
\end{tabular}
\end{table}



\subsubsection{\autoref{tab:hutchinson_experiment}}
The script for training Spiral, Moons, GMM, Swissroll, MNIST and CIFAR-10 are provided in the supplementary code.

Noticed that technically, the weighting function needs the Hessian of data likelihood for each noise level. The approximation here is actually trace of hessian computed from model using Hutchinson estimator probes.

The implementation details can be found in the supplementary code.

\subsubsection{\autoref{fig:variance of parameter gradients dim 10}}

The source of the Gaussian distribution yields a constant Hessian, which undermines the weighting function we have and therefore, we plot \autoref{fig:variance of parameter gradients dim 10} using the source distribution of a mixture of Gaussians. The mean is fixed at $[0, 2, 4]$, and standard deviation and weights for each Gaussian are generated randomly with seed 0. Code for details can be found in the supplementary. 

The data distribution is defined as:
\begin{align*}
    p(x_0) = 0.3258 \times \mathcal{N}(0, 0.5063^2) + 0.3316 \times \mathcal{N}(2.0, 0.7782^2) + 0.3426 \times \mathcal{N}(4, 0.0985^2).
\end{align*}

The resulting marginal density \(p(x_t)\), obtained by convolving \(p(x_0)\) with the Gaussian kernel, is given by:
{\footnotesize\begin{align*}
    p(x_t) = 0.3258 \times \mathcal{N}(0, 0.5063^2 + \sigma_t^2) + 0.3316 \times \mathcal{N}(2.0, 0.7782^2 + \sigma_t^2) + 0.3426 \times \mathcal{N}(4, 0.0985^2 + \sigma_t^2).
\end{align*}}

We derive the analytical expressions for both the score function and the Hessian of \(\log p(x_t)\), and use the reverse-time SDE to generate samples from the model.

Models are trained using both the optimal and heuristic weighting functions separately. After training, we generate 5000 samples from each model using the reverse SDE. For comparison, we also generate 5000 samples using the true score function to serve as a ground-truth reference and to isolate any bias introduced by the reverse SDE sampling process.


\section{Proofs for Theorem and Lemmas}
\label{sec: proof}

\subsection{Proof of Definition \ref{def:first-order-score}}
\begin{proof}\renewcommand{\qedsymbol}{} \label{prf: first-order-score}
Given the first order score estimator $s(x_t) = \E_{x_0|x_t}[s(x_t|x_0)]$
\begin{align}
    s_2(x_t) &= \nabla_{x_t} s(x_t) \\
    &=\nabla_{x_t} \E_{x_0|x_t}[s(x_t|x_0)] \\
    &=\nabla_{x_t} \int s(x_t|x_0) p(x_0|x_t) dx_0 \\
    &= \int \left(\nabla_{x_t} s(x_t|x_0) p(x_0|x_t) + s(x_t|x_0) \nabla p(x_t|x_t)^\top\right) dx_0 \\
    &= \int \left(\nabla_{x_t} s(x_t|x_0) p(x_0|x_t) + s(x_t|x_0) \nabla \log p(x_0|x_t)^\top p(x_0|x_t)\right) dx_0 \\
    &= \E_{x_0|x_t} \left[\nabla_{x_t} s(x_t|x_0) + s(x_t|x_0)\nabla_{x_t} \log p(x_t|x_0)^\top - s(x_t|x_0)\nabla_{x_t} \log p(x_t)^\top\right]
\end{align}
    The last equation is obtained from the use of Bayes' rule to expand the posterior probability of $p(x_0|x_t)$ 
\end{proof}

\subsection{Proof of \autoref{thm:iterative-score}}
\begin{proof}\renewcommand{\qedsymbol}{} \label{Prf: thm:iterative-score}
We prove Theorem 4.1 by Mathematical Induction.
Given $s_{k-1} = \E_{x_0|x_t} \left[h_{k-1}(x_0, x_t)\right]$, we take derivatives of $x_t$ from both sides, which yields
\begin{align}
    s_{k}(x_t) &= \nabla_{x_t} s_{k-1}(x_t) \\
    &= \nabla_{x_t} \E_{x_0|x_t}\left[h_{k-1}(x_0, x_t)\right] \\
    &= \nabla_{x_t} \int  h_{k-1}(x_0, x_t) p(x_0|x_t) dx_0 \\
    &= \int \nabla_{x_t}h_{k-1}(x_0, x_t) p(x_0|x_t) + h_{k-1}(x_0, x_t) \nabla_{x_t} p(x_0|x_t)^\top dx_0 \\
    &= \int \nabla_{x_t}h_{k-1}(x_0, x_t) p(x_0|x_t) + h_{k-1}(x_0, x_t) \nabla_{x_t} \frac{p(x_0|x_t)^\top}{p(x_0|x_t)}p(x_0|x_t) dx_0 \\
    &= \int \nabla_{x_t}h_{k-1}(x_0, x_t) p(x_0|x_t) + h_{k-1}(x_0, x_t) \nabla_{x_t} \log p(x_0|x_t)^\top p(x_0|x_t) dx_0 \\
    &= \E_{x_0|x_t} \left[\nabla_{x_t} h_{k-1}(x_0, x_t) + h_{k-1}(x_0, x_t) \nabla_{x_t} \log p(x_0|x_t)^\top\right] \\
    &= \E_{x_0|x_t} \left[\nabla_{x_t} h_{k-1}(x_0, x_t) + h_{k-1}(x_0, x_t) \nabla_{x_t} \log p(x_t|x_0)^\top - h_{k-1}(x_0, x_t) \nabla_{x_t} \log p(x_t)^\top \right] \label{eq: key steps to introduce operator} \\
    &= \E_{x_0|x_t} \left[h_k(x_0, x_t)\right]
\end{align}
\end{proof}

\autoref{eq: key steps to introduce operator} uses  Bayes' rules to expand the posterior density $p(x_0|x_t)$
$$p(x_0|x_t) = \frac{p(x_t|x_0)p(x_0)}{p(x_t)}$$
If we take the derivative of the logarithm of the equation with respect to $x_t$ from both sides
$$\nabla_{x_t} \log p(x_0|x_t) = \nabla_{x_t} \log p(x_t|x_0) - \nabla_{x_t} \log p(x_t)$$

\subsection{Proof of \autoref{thm:loss-decomposition}}
\begin{proof}\renewcommand{\qedsymbol}{}
\label{prf: loss-decomposition}
Given the Monte Carlo estimator of $k^{\text{th}}$ order score function $s_k(x_t) = \E_{x_0|x_t}\left[h_k(x_0, x_t)\right]$, we apply $L_2$ norm to enforce pointwise convergence between a parametrized score function the true score,
\begin{align}
    \mathcal{L}_{\text{dsm}}
    &=\E_{p(x_0, x_t)}\Bigl\|s_k(x_t;\theta)-h_k(x_0, x_t)\Bigr\|_2^2 \\
    &= \E_{p(x_0, x_t)}\Bigl\| s_k(x_t; \theta) - \E_{x_0|x_t}(h_k(x_0, x_t)) +  \E_{x_0|x_t}(h_k(x_0, x_t)) - h_k(x_0, x_t)\Bigr\|_2^2  \\
    &= \E_{p(x_0, x_t)}\Bigl\| s_k(x_t; \theta) - \E_{x_0|x_t}(h_k(x_0, x_t)) \Bigr\|_2^2 + \notag \\
    & \E_{p(x_0, x_t)}\Bigl\| h_k(x_0, x_t) - \E_{x_0|x_t}(h_k(x_0, x_t))   \Bigr\|_2^2 + \notag \\
    & \quad 2 \E_{p(x_0, x_t)} \Bigl[\Bigl(s(x_t; \theta) - \E_{x_0|x_t}(h_k(x_0, x_t)) \Bigl) \Bigl( h_k(x_0, x_t) - \E_{x_0|x_t}(h_k(x_0, x_t) \Bigl) \Bigr] \\
    &= \E_{p(x_0, x_t)}\Bigl\| s_k(x_t; \theta) - s_k(x_t) \Bigr\|_2^2 + \E_{p(x_0, x_t)}\Bigl\| h_k(x_0, x_t) - \E_{x_0|x_t}(h_k(x_0, x_t))   \Bigr\|_2^2 + \notag \\
    & \quad 2 \E_{p(x_0, x_t)} \Bigl[\Bigl(s(x_t; \theta) - s_k(x_t) \Bigl) \Bigl( h_k(x_0, x_t) - s_k(x_t) \Bigr) \Bigr] \\
    &= \mathcal{L}_{\text{sm}} + \E_{x_t} \E_{x_0|x_t} \Bigl\| h_k(x_0, x_t) - \E_{x_0|x_t}(h_k(x_0, x_t))   \Bigr\|_2^2 + \notag  \\
    & \quad 2 \E_{p(x_t)}  \Bigl[\Bigl(s(x_t; \theta) - s_k(x_t) \Bigl)  \Bigl( \E_{x_0|x_t}(h_k(x_0, x_t)) - s_k(x_t) \Bigr) \Bigr] \\
    &= \mathcal{L}_{\text{sm}} + \tr\;\E_{x_t} \mathbb V_{x_0|x_t} (h_k(x_0, x_t))
\end{align}

\end{proof}

\subsection{Proof of \ref{lem: sandwich form for variance}}
\begin{proof}\renewcommand{\qedsymbol}{}
\label{prf: sandwich form for variance}
    First, let us restate the weighted DSM loss function for the sake of completeness.
    \begin{align}
         \mathcal{L}_{dsm} = \frac{1}{N}\sum_{x_0, x_t}\big(s_k(x_t;\theta)-h_k(x_0,x_t)\big)^\top W \big(s_k(x_t;\theta)-h_k(x_0,x_t)\big) 
    \end{align}
Similarly, the weighted SM loss is 
    \begin{align}
         \mathcal{L}_{sm} = \frac{1}{N}\sum_{x_0, x_t}\Big(\big(s_k(x_t;\theta)-s(x_t)\big)^\top W \big(s_k(x_t;\theta)-s(x_t)\big)\Big) 
    \end{align}

    The gradient of the DSM loss with respect to the model parameter $\theta$ is
    \begin{align}
        g_{\text{dsm}} &= \frac{1}{N} \sum_{x_0, x_t}\frac{\partial \mathcal{L}_{\text{dsm}}}{\partial \theta}  \\
     &= \frac{1}{N} \sum_{x_0, x_t}\Big ( 2 J^\top(x_t) W [s_k(x_t; \theta) - h_k(x_0, x_t)] \Big)\\
    \end{align}
    where $J^{\top}(x_t) = \partial_\theta s(x_t; \theta)$, i.e, the Jacobian matrix of the model wrt the parameters.
Similarly, the gradient of SM loss wrt the model parameter $\theta$ is 
    \begin{align}
        g_{\text{sm}} &= \frac{1}{N} \sum_{x_0, x_t}\frac{\partial \mathcal{L}_{\text{sm}}}{\partial \theta}  \\
     &= \frac{1}{N} \sum_{x_t}\Big ( 2 J^\top(x_t) W [s_k(x_t; \theta) - s_k(x_t)] \Big)\\
    \end{align}



    \paragraph{Variance of the gradient compared to Score Matching}
    The variance of the gradients given a batch of data $x_t$ is
    {\footnotesize \begin{align}
        \mathbb{V}_{x_0, x_t} (g) &= \frac{1}{N} \sum_{x_0, x_t} \mathbb{V}(g_{x_0, x_t})
    \end{align}
    Inside the summation is the gradient for $\mathcal{L}$ wrt each datum $x_0, x_t$, which we denoted as $g_{x_0, x_t}$
    \begin{align}
        &\mathbb{V}(g_{x_0, x_t}) \notag \\
        &= \E(g_{x_0, x_t}g^\top_{x_0, x_t}) - \E(g_{x_0, x_t})\E(g_{x_0, x_t})^\top \\
        &= \E \bigg( J^\top W(s_k(x_t; \theta) - h_k(x_0, x_t))(s_k(x_t; \theta) - h_k(x_0, x_t))^\top J \bigg) - \notag \\
        & \E\Big(J^\top (s(x_t; \theta) - h_k(x_0, x_t)) \Big)\E\Big(J^\top (s_k(x_t; \theta) - h_k(x_0, x_t)) \Big)^\top \\
        &= \E \bigg( J^\top W(s_k(x_t; \theta) - s_k(x_t) + s_k(x_t) - h_k(x_0, x_t))(s_k(x_t; \theta) - s_k(x_t) + s_k(x_t) - h_k(x_0, x_t))^\top J \bigg) -  \notag \\
        &\qquad \E_{x_t}\Big(J^\top (s_k(x_t; \theta) - s_k(x_t)) \Big)\E_{x_t}\Big(J^\top (s_k(x_t; \theta) - s_k(x_t)) \Big)^\top \\
        &= \E \bigg(J^\top W \Big(s_k(x_t; \theta) - s_k(x_t) \Big)\Big(s_k(x_t; \theta) - s_k(x_t) \Big)^\top W^\top J \bigg) + \notag \\
        & \underbrace{\E \bigg(J^\top W \Big(s_k(x_t) - h_k(x_0, x_t) \Big)\Big(s_k(x_t; \theta) - s_k(x_t) \Big)^\top W^\top J \bigg)}_{0} +  \notag \\
        &\quad   \underbrace{\E \bigg(J^\top W \Big(s_k(x_t; \theta) - s_k(x_t) \Big)\Big(s_k(x_t) - h_k(x_0, x_t) \Big)^\top W^\top J \bigg)}_{0} +  \notag \\
        & \qquad \underbrace{\E \bigg(J^\top W \Big(s_k(x_t) - h_k(x_0, x_t) \Big)\Big(s_k(x_t) - h_k(x_0, x_t) \Big)^\top W^\top J \bigg)}_{\text{heteroscedasticity}}   \notag \\
        &\qquad - \E_{x_t}\Big(J^\top (s_k(x_t; \theta) - s_k(x_t)) \Big)\E_{x_t}\Big(J^\top (s_k(x_t; \theta) - s_k(x_t)) \Big)^\top \\
        &=\E_{x_t} \bigg(J^\top W \Big(s_k(x_t; \theta) - s_k(x_t) \Big)\Big(s_k(x_t; \theta) - s_k(x_t) \Big)^\top W^\top J \bigg) - \notag \\
        & \E_{x_t}\Big(J^\top (s_k(x_t; \theta) - s_k(x_t)) \Big)\E_{x_t}\Big(J^\top (s_k(x_t; \theta) - s_k(x_t)) \Big)^\top + \notag \\
        & \qquad \underbrace{\E \bigg(J^\top W \Big(s_k(x_t) - h_k(x_0, x_t) \Big)\Big(s_k(x_t) - h_k(x_0, x_t) \Big)^\top W^\top J \bigg)}_{\text{heteroscedasticity}} \\
        &= \E(g_{\text{sm}}g^\top_{\text{sm}}) - \E(g_{\text{sm}})\E(g_{\text{sm}})^\top + \underbrace{\E \bigg(J^\top W \Big(s_k(x_t) -h_k(x_0, x_t) \Big)\Big(s_k(x_t) - s_k(x_t) \Big)^\top W^\top J \bigg)}_{\text{heteroscedasticity}} \\
        &=  \mathbb{V}(g_{\text{sm}}) + \underbrace{\E \bigg(J^\top W \Big(s_k(x_t) - h_k(x_0, x_t) \Big)\Big(s_k(x_t) - h_k(x_0, x_t) \Big)^\top W^\top J \bigg)}_{\text{heteroscedasticity}} \\
        &= \mathbb{V}(g_{\text{sm}}) + \E_{x_t} \Big( J^\top W \mathbb{V}_{x_0|x_t}(h_k(x_0, x_t)) W^\top J \Big)
    \end{align}
    }
    
    The variance of the gradient for DSM is related to SM with an extra term that depends on $x_t$ and $\sigma_t$ if $W$ doesn't depend on $x_t$ and $t$. 

    \paragraph{Variance of the parameter}
    The variance of parameters at convergence can be obtained by directly using the sandwich estimator or Godambe Information~\citep{godambe1991estimating}. 

    At convergence, we have
    $$g_{\text{dsm}}(\hat{\theta}) = 0$$

    For large number of dataset when $\hat{\theta}$ is close to $\theta^{*}$, we have
    $$0 = g(\hat{\theta}) \approx g(\theta^{*}) + H_{\mathcal{L}}(\theta^{*})(\hat{\theta} - \theta^{*})$$

    Therefore, the 
    $$\hat{\theta} - \theta^{*} = H_{\mathcal{L}}(\hat{\theta})g_{\text{dsm}}(\theta^{*})$$

    Finally, the variance of the parameter is obtained
    \begin{align}
        V(\theta) = \E_{x_t} \Big(S(x_t, \theta^{*})G(x_t, t) S(x_t, \theta^{*}) \Big)
    \end{align}
    where 
    \begin{align}
        \begin{cases}
            H_{\mathcal{L}}(\theta^{*}) &= 2 J^\top(x_t; \theta^{*}) W J(x_t; \theta^{*}) \\
            S(x_t; \theta^{*}) &= H_{\mathcal{L}}(x_t, \theta^{*})^{-1} J(x_t; \theta^{*})^\top \\
            G(x_t, t) &= W \mathbb{V}_{x_0|x_t} (h_k(x_t, x_0)) W^\top
        \end{cases}
    \end{align}
\end{proof}

\subsection{Proof of \cref{lemma: practical approximation of weighting function}}
\begin{proof}\renewcommand{\qedsymbol}{}
\label{prf: practical approx}
We begin with the ideal weighting function again:
\begin{align}
 W &= \Bigl(\mathbb{V}_{x_0|x_t}(h_k(x_t, x_0)) \Bigr)^{-1} \\
 &= \Bigl(\E_{x_0}\fisher^{(k)}_{x_t|x_0} - \fisher^{(k)}_{x_t} \Bigr)^{-1}
\end{align}

Let $\mathcal{A} = \E_{x_0|x_t}[\fisher^{(k)}_{x_t|x_0}]$ and $\mathcal{B} = \fisher^{(k)}_{x_t}$ represent the corresponding rank-$2k$ Fisher information tensors, acting as linear operators on the space of rank-$k$ score tensors. 

Therefore with Neumann Series yields,
\begin{align}
    W &= (A - B)^{-1} \\
      &= (I - A^{-1}B)^{-1} A^{-1} \\
      &= (I + A^{-1}B + (A^{-1}B)^2 + ...) A^{-1} \\
      &\approx A^{-1} \\
      &= \Bigl( \E_{x_0} \fisher_{x_t|x_0} \Bigr)^{-1}
\end{align}
    
\end{proof}

\subsection{Proof of Lemma \cref{variance properties of approximated weights}}
\begin{proof}\renewcommand{\qedsymbol}{}
\label{prf: variance properties of approx}
To compare the variance of the gradients, we establish the relationship between the weighting operators $W_1$ and $W_2$. By the law of total variance and the properties of the Fisher information for $k$-th order scores, the conditional variance tensor of the score can be decomposed as:
\begin{align}
\mathbb{V}{x_0|x_t}(h_k(x_t, x_0)) &= \E_{x_0}\fisher^{(k)}_{x_t|x_0} - \fisher^{(k)}_{x_t}
\end{align}

Let $\mathcal{A} = \E_{x_0|x_t}[\fisher^{(k)}_{x_t|x_0}]$ and $\mathcal{B} = \fisher^{(k)}_{x_t}$ represent the corresponding rank-$2k$ Fisher information tensors, acting as linear operators on the space of rank-$k$ score tensors. 

Because $\mathcal{B}$ is a Fisher information tensor, it is positive semidefinite ($\mathcal{B} \succeq 0$). Therefore, we can rewrite our weighting operators as:
\begin{align}
W_1 &= \mathcal{A}^{-1} \quad W_2= (\mathcal{A} - \mathcal{B})^{-1} \quad \text{and} \quad \mathbb{V}_{x_0|x_t}(s^{(k)}(x_t|x_0)) = \mathcal{A} - \mathcal{B}
\end{align}

Substituting these into the irreducible penalty terms for both gradient variances:

For the approximated weights, the inner operator composition becomes:

\begin{align}
W_1 \mathbb{V}_{x_0|x_t}(s^{(k)}(x_t|x_0)) W_1^* &= \mathcal{A}^{-1} (\mathcal{A} - \mathcal{B}) \mathcal{A}^{-1}
\end{align}

Note that because $\mathcal{A}$ is a self-adjoint covariance operator, its adjoint $\mathcal{A}^* = \mathcal{A}$, meaning $W_1^* = W_1$.

For the ideal weights, the inner composition simplifies to:
\begin{align}
W_2 \mathbb{V}_{x_0|x_t}(s^{(k)}(x_t|x_0)) W_2^* &= (\mathcal{A} - \mathcal{B})^{-1} (\mathcal{A} - \mathcal{B}) (\mathcal{A} - \mathcal{B})^{-1} \notag \&= (\mathcal{A} - \mathcal{B})^{-1} = W_2\end{align}

Recall that for any self-adjoint, positive definite operator $\mathcal{A}$ and positive semidefinite operator $\mathcal{B}$ such that $\mathcal{A} - \mathcal{B} \succ 0$, the Loewner partial order extends to operators on finite-dimensional Hilbert spaces, guaranteeing that:
\begin{align}(\mathcal{A} - \mathcal{B})^{-1} &\succeq \mathcal{A}^{-1} (\mathcal{A} - \mathcal{B}) \mathcal{A}^{-1}\end{align}This implies that $W_2 \succeq W_1 \mathbb{V}_{x_0|x_t}(h_k(x_t, x_0)) W_1^*$.

Let $J(x_t; \theta)$ be the Jacobian of the neural network output with respect to the parameters $\theta$, mapping from the parameter space to the rank-$k$ tensor space. Pre-multiplying by its adjoint $J(x_t; \theta)^*$ and post-multiplying by $J(x_t; \theta)$ via tensor contraction preserves this positive semidefinite ordering in the parameter space. Taking the expectation over $x_t$ yields:\begin{align}\E_{x_t} \bigg(J(x_t; \theta)^* W_2 J(x_t; \theta) \bigg) \succeq \E_{x_t} \bigg(J(x_t; \theta)^* W_1 \mathbb{V}{x_0|x_t}(s^{(k)}(x_t|x_0)) W_1^* J(x_t; \theta) \bigg)
\end{align}
Consequently, comparing the heteroscedastic penalty terms, we conclude that the gradient variance adheres to the strict ordering:
\begin{align}
V(g_{\text{dsm}})_{\text{ideal}} \succeq V(g_{\text{dsm}})_{\text{approx}}
\end{align}

\end{proof}

\subsection{Proof of Lemma \ref{lemma: weighting function}}
\label{prf: hetero analysis}
When the model is parameterized as linear, this means
$$s(x_t; \theta) = \phi(x_t, t)\theta$$
where $\phi$ is the kernel feature of $x_t$ and $t$, such as polynomial kernel.

In this case, the Jacobian matrix is $J(x_t; \theta) = \partial_\theta s(x_t; \theta) = \phi(x_t, t)$.

In addition, the Hessian of the loss function is reduced to
\begin{align}
    H_{\mathcal{L}}(\theta^{*}) = 2 J^\top W J = 2 \phi(x_t, t)^\top W\phi(x_t, t)
\end{align}

So the variance of the estimator in Lemma \ref{lem: sandwich form for variance} becomes 
$$\mathbb{V}(\theta) = \E\bigg( \Big(\phi(x_t,  t)^\top W \phi(x_t, t)\Big)^{-1} \phi(x)^\top G(x_t, t) \phi(x) \Big(\phi(x_t,  t)^\top W \phi(x_t, t)\Big)^{-1} \bigg)$$

The rest follows the proof of the Gauss-Markov Theorem to show Godambe weight yields the least variance unbiased estimator. We refer the reader to Section 10.2.2 for the detailed proof~\cite{shalizi2013advanced}.

\paragraph{
Comparison of the $\mathbb{V}(g)$ and $\mathbb{V}(\theta)$ regardless of the orders and noise type when the model is linear.
}
We refer readers to~\cite{godambe1991estimating} since the proof is very general.

\begin{lemma}[Upper‐Bound via Stam’s Convolution Inequality]\label{lem:true-weight-upper}
Let $I_0 = \tr\!\bigl(\fisher(p_{X_0})\bigr)<\infty$.  Then for all $t\ge0$ one has by Stam’s inequality
\[
\tr\!\bigl(\fisher(p_{X_t})\bigr)
\;\le\;\Bigl(\tfrac1{\fisher_0} + \tfrac{\sigma_t^2}{D}\Bigr)^{-1},
\]
and hence, since 
\(
q=\tfrac{\sigma_t^2}{D}\,\tr(\fisher(p_{X_t})) 
\) 
and 
\(
w=\tfrac{\sigma_t^2}{D(1-q)},
\)
\[
w
\;\le\;
\sigma_t^2 
\;+\;\frac{\fisher_0}{D}\,\sigma_t^4.
\]
This equality holds if and only if $X_0$ is normally distributed.
\end{lemma}
\begin{proof}\renewcommand{\qedsymbol}{}[Proof of Lemma \autoref{lem:true-weight-upper}] \label{prf: Proof of Lemma using stam's ineq}
 Stam's convolution inequality defines the relationship between the Fisher information of the summation of two independent random variables and their Fisher information:
 \begin{align}
     (I(X + Y))^{-1} \leq I(X)^{-1} + I(Y)^{-1}
 \end{align}
which can be directly applied to our weighting function to find the upper bound of the weighting function $w$
 \begin{align}
     w &= \left(\frac{I}{\sigma_t^2} - I(\sigma_t)\right)^{-1} \\
     &\leq \left(\frac{I}{\sigma_t^2} - I(x_0 + \sigma_t\epsilon)\right)^{-1} \\
     &=\left(\frac{I}{\sigma_t^2} - \left(I(x_0)^{-1} + \sigma_t^2\right)^{-1}\right)
 \end{align}
\end{proof}

\section{Analysis of second order estimators} \label{sec: analysis of second order estimator}

We provide an analysis for 2 estimators proposed in existing literature~\cite{meng2021estimating, lu2022maximum} and the estimator proposed above from a bias and variance view.
\begin{theorem}
    Given the first-order score function $s(x_t)$, the following Monte Carlo estimators can be used to estimate the second-order score function.
    \begin{equation}
        \begin{cases}
            T_1 = E_{x_0| x_t} \left[s_1(x_t|x_0)s_1(x_t|x_0)^\top - s_1(x_t|x_0) s_1(x_t)^\top + s_2(x_t|x_0)\right] \\
            T_2 = E_{x_0|x_t} \left[(s_1(x_t|x_0) - s_1(x_t))(s_1(x_t|x_0) - s_1(x_t))^\top + s_2(x_t|x_0) \right] \quad \text{\citet{lu2022maximum}}\\
            T_3 = E_{x_0| x_t} \left[s(x_t|x_0)s(x_t|x_0)^\top - s(x_t) s(x_t)^\top + s_2(x_t|x_0)\right] \quad \text{\citet{meng2021estimating}}
        \end{cases}
    \end{equation}
\end{theorem}

\subsection{Comparison of different estimators}
\subsubsection{Our estimator}
Given that the error of the final estimator $s(\bar{x}; \theta^{(t)})$ from iteration $t$ is uncorrelated with our estimator $\delta \perp s(\bar{x})$
\begin{align}
    T_1(\theta^{(t)}) &= E_{x| \bar{x}} \left[s(\bar{x}|x)s(\bar{x}|x)^\top - s(\bar{x}|x) s(\bar{x}; \theta^{(t)})^\top + H(\bar{x}|x)\right] \\
     &= E_{x|\bar{x}} \left[ s(\bar{x}|x)s(\bar{x}|x)^\top - s(\bar{x}|x) s(\bar{x})^\top + H(\bar{x}|x) - s(\bar{x}|x)\delta^\top \right] \\
     &= H(\bar{x}) - s(\bar{x})\delta^\top
\end{align}
The expectation of our estimator $T_1$ is
\begin{align}
    E_{\bar{x}} [T_1(\theta)] &= E_{\bar{x}} \left[ H(\bar{x}) - s(\bar{x})\delta\right] \\
    & = E_{\bar{x}} \left[ H(\bar{x}) \right] - E_{\bar{x}} \left[s(\bar{x}) \delta\right] \\
    &= E_{\bar{x}} \left[ H(\bar{x}) \right] + \delta \delta^\top
\end{align}
Thus, $T_1$ is unbiased at convergence when the first-order score estimator is exact ($\delta=0$).

The variance of our estimator $T_1$ is
\begin{align}
    &E_{\bar{x}}(T_1 - E_{\bar{x}}(T_1))(T_1 - E_{\bar{x}}(T_1))^\top \\
      &= Cov(H(\bar{x})) - E_{\bar{x}} \left(s(\bar{x})\delta^\top H(\bar{x})^\top\right) - E_{\bar{x}} \left(H(\bar{x})\delta s(\bar{x})^\top\right) + E(s(\bar{x})\delta^\top\delta s(\bar{x})^\top)
\end{align}

\subsubsection{Estimator from \citet{lu2022maximum}}
Given that the error of the current estimator $s(\bar{x}; \theta)$  is $\delta = s(\bar{x}; \theta) - s(\bar{x})$,
\begin{align}
    T_2(\theta) &= E_{x|\bar{x}} \left[(s(\bar{x}|x) - s(\bar{x}; \theta)(s(\bar{x}|x) - s(\bar{x}; \theta)^\top + H(\bar{x}|x) \right] \\
    &= E_{x|\bar{x}} \left[(s(\bar{x}|x) - s(\bar{x}) - \delta)(s(\bar{x}|x) - s(\bar{x}) - \delta)^\top + H(\bar{x}|x) \right] \\
    &= H(\bar{x}) + \delta\delta^\top
\end{align}

The expectation is 
\begin{align}
    E_{\bar{x}} (T_2(\theta^{(t)})) &= E_{\bar{x}}(H(\bar{x})) + \delta\delta^\top
\end{align}
So the estimator is biased when the score function is biased.

The variance is 
\begin{align}
    & E_{\bar{x}}(T_2 - E_{\bar{x}}(T_2))(T_2 - E_{\bar{x}}(T_2))^\top \\
    & = Cov(H(\bar{x}))
\end{align}
The variance is small.


    

\subsubsection{Estimator from \citet{meng2021estimating}}
Given that the error of the current estimator $s(\bar{x}; \theta)$ from iteration $t$ is $\delta = s(\bar{x}; \theta) - s(\bar{x})$,
\begin{align}
    T_3(\theta) &= E_{x| \bar{x}} \left[s(\bar{x}|x)s(\bar{x}|x)^\top - s(\bar{x}; \theta) s(\bar{x}; \theta)^\top + H(\bar{x}|x)\right] \\
    &= E_{x| \bar{x}} \left[s(\bar{x}|x)s(\bar{x}|x)^\top - s(\bar{x})s(\bar{x}) + H(\bar{x}|x) - \delta s(\bar{x})^\top - s(\bar{x})\delta^\top -\delta\delta^\top\right] \\
    &= H(\bar{x}) - \delta\delta^\top - \delta s(\bar{x})^\top - s(\bar{x})\delta^\top
\end{align}
The expectation is
\begin{align}
    E_{\bar{x}}(T_3(\theta^{(t)})) &= E_{\bar{x}}\left[ H(\bar{x}) - \delta\delta^\top - \delta s(\bar{x})^\top - s(\bar{x})\delta^\top\right] \\
    &= E_{\bar{x}}\left[ H(\bar{x})\right] - \delta\delta^\top
\end{align}
The variance is 
\begin{align}
    & E_{\bar{x}}(T_3 - E_{\bar{x}}(T_3))(T_3 - E_{\bar{x}}(T_3))^\top \\
    &= Cov(H(\bar{x})) - E_{\bar{x}} \left[H(\bar{x})s(\bar{x})\delta^\top\right] - E_{\bar{x}}\left[H(\bar{x})\delta s(\bar{x})^\top\right] - E_{\bar{x}} \left[\delta s(\bar{x})^\top H(\bar{x})^\top \right] - \notag \\
    & E_{\bar{x}} \left[ s(\bar{x})\delta^\top H(\bar{x})^\top  \right] + \delta E_{\bar{x}}[s(\bar{x})^\top s(\bar{x})] \delta^\top - E_{\bar{x}} [\delta s(\bar{x})^\top \delta s(\bar{x})^\top] - \notag \\
    & E_{\bar{x}} [s(\bar{x})\delta^\top s(\bar{x})\delta^\top] + E_{\bar{x}} [s(\bar{x})\delta^\top\delta s(\bar{x})^\top]
\end{align}

\begin{remark}
From the perspective of bias, given that the first-order score function $s_1(x_t)$ is uncorrelated with error, our estimator $T_1$ is unbiased when the first-order score error vanishes. $T_2$ is biased positively up to the quadratic form of the error. $T_3$ is biased negatively up to the quadratic form of the error. 

From the perspective of variances, given that the first-order score function $s_1(x_t)$ is uncorrelated with error, $T_2$ is the estimator with minimal variance.
\end{remark}

\section{Training Details}
For all the experiments, we use an encoder/decoder structure from~\cite{lu2022maximum} with 2 hidden layers of 128 hidden units in the encoder and 2 hidden layers of 256 hidden units in the decoder. The optimizer is Adam. The total number of iterations is 250k. The code is provided in the supplementary.

\vfill



\newpage
\section*{NeurIPS Paper Checklist}

The checklist is designed to encourage best practices for responsible machine learning research, addressing issues of reproducibility, transparency, research ethics, and societal impact. Do not remove the checklist: {\bf The papers not including the checklist will be desk rejected.} The checklist should follow the references and follow the (optional) supplemental material.  The checklist does NOT count towards the page
limit. 

Please read the checklist guidelines carefully for information on how to answer these questions. For each question in the checklist:
\begin{itemize}
    \item You should answer \answerYes{}, \answerNo{}, or \answerNA{}.
    \item \answerNA{} means either that the question is Not Applicable for that particular paper or the relevant information is Not Available.
    \item Please provide a short (1--2 sentence) justification right after your answer (even for \answerNA). 
\end{itemize}

{\bf The checklist answers are an integral part of your paper submission.} They are visible to the reviewers, area chairs, senior area chairs, and ethics reviewers. You will also be asked to include it (after eventual revisions) with the final version of your paper, and its final version will be published with the paper.

The reviewers of your paper will be asked to use the checklist as one of the factors in their evaluation. While \answerYes{} is generally preferable to \answerNo{}, it is perfectly acceptable to answer \answerNo{} provided a proper justification is given (e.g., error bars are not reported because it would be too computationally expensive'' or ``we were unable to find the license for the dataset we used''). In general, answering \answerNo{} or \answerNA{} is not grounds for rejection. While the questions are phrased in a binary way, we acknowledge that the true answer is often more nuanced, so please just use your best judgment and write a justification to elaborate. All supporting evidence can appear either in the main paper or the supplemental material, provided in appendix. If you answer \answerYes{} to a question, in the justification please point to the section(s) where related material for the question can be found.

IMPORTANT, please:
\begin{itemize}
    \item {\bf Delete this instruction block, but keep the section heading ``NeurIPS Paper Checklist"},
    \item  {\bf Keep the checklist subsection headings, questions/answers and guidelines below.}
    \item {\bf Do not modify the questions and only use the provided macros for your answers}.
\end{itemize}


\begin{enumerate}

\item {\bf Claims}
    \item[] Question: Do the main claims made in the abstract and introduction accurately reflect the paper's contributions and scope?
    \item[] Answer: \answerYes{}
    \item[] Justification: We theoretically derive the optimal weighting function for denoising diffusion
loss functions and discuss the tradeoff between statistical optimality and optimization dynamics.

\item {\bf Limitations}
    \item[] Question: Does the paper discuss the limitations of the work performed by the authors?
    \item[] Answer: \answerYes{}
    \item[] Justification: Our limitations section points to assumptions and requirements for the optimal
weighting function to be computed

\item {\bf Theory assumptions and proofs}
    \item[] Question: For each theoretical result, does the paper provide the full set of assumptions and a complete (and correct) proof?
    \item[] Answer: \answerYes{}
    \item[] Justification: All assumptions are clearly stated, theorems and lemmas are numbered and
cross references

    \item {\bf Experimental result reproducibility}
    \item[] Question: Does the paper fully disclose all the information needed to reproduce the main experimental results of the paper to the extent that it affects the main claims and/or conclusions of the paper (regardless of whether the code and data are provided or not)?
    \item[] Answer: \answerYes{}
    \item[] Justification: Our paper is predominantly theoretical, but the small number empirical
experiments included are easily reproduced following the guidelines in the supplementary material.

\item {\bf Open access to data and code}
    \item[] Question: Does the paper provide open access to the data and code, with sufficient instructions to faithfully reproduce the main experimental results, as described in supplemental material?
    \item[] Answer: \answerYes{}.
    \item[] Justification: Our contributions are primarily theoretical, and instructions on reproducingempirical experiments included in the supplementary material.

\item {\bf Experimental setting/details}
    \item[] Question: Does the paper specify all the training and test details (e.g., data splits, hyperparameters, how they were chosen, type of optimizer) necessary to understand the results?
    \item[] Answer: \answerYes{}.
    \item[] Justification: The supplementary material includes details on experiments supporting our theory.

\item {\bf Experiment statistical significance}
    \item[] Question: Does the paper report error bars suitably and correctly defined or other appropriate information about the statistical significance of the experiments?
    \item[] Answer:\answerYes{}
    \item[] Justification: Our evaluation is reproduced multiple times and reported error bars where suitable.

\item {\bf Experiments compute resources}
    \item[] Question: For each experiment, does the paper provide sufficient information on the computer resources (type of compute workers, memory, time of execution) needed to reproduce the experiments?
    \item[] Answer: \answerYes{} 
    \item[] Justification: We have detailed the information in the appendix.
    
\item {\bf Code of ethics}
    \item[] Question: Does the research conducted in the paper conform, in every respect, with the NeurIPS Code of Ethics \url{https://neurips.cc/public/EthicsGuidelines}?
    \item[] Answer: \answerYes{} 
    \item[] Justification: We have read the ethics guidelines and believe that we comply with these.

\item {\bf Broader impacts}
    \item[] Question: Does the paper discuss both potential positive societal impacts and negative societal impacts of the work performed?
    \item[] Answer: \answerNA{} 
    \item[] Justification: Our paper provides a theoretical analysis of an existing approach used in score matching, and direct societal impact is limited.
    
\item {\bf Safeguards}
    \item[] Question: Does the paper describe safeguards that have been put in place for responsible release of data or models that have a high risk for misuse (e.g., pre-trained language models, image generators, or scraped datasets)?
    \item[] Answer: \answerNA{} 
    \item[] Justification: Our paper is theoretical in nature, and safeguards are not required.

\item {\bf Licenses for existing assets}
    \item[] Question: Are the creators or original owners of assets (e.g., code, data, models), used in the paper, properly credited and are the license and terms of use explicitly mentioned and properly respected?
    \item[] Answer: \answerNA{} 
    \item[] Justification: No existing assets are used.

\item {\bf New assets}
    \item[] Question: Are new assets introduced in the paper well documented and is the documentation provided alongside the assets?
    \item[] Answer: \answerNA{} 
    \item[] Justification: No new assets are released.

\item {\bf Crowdsourcing and research with human subjects}
    \item[] Question: For crowdsourcing experiments and research with human subjects, does the paper include the full text of instructions given to participants and screenshots, if applicable, as well as details about compensation (if any)? 
    \item[] Answer: \answerNA{} 
    \item[] Justification: No crowd-sourcing/ human research was conducted.

\item {\bf Institutional review board (IRB) approvals or equivalent for research with human subjects}
    \item[] Question: Does the paper describe potential risks incurred by study participants, whether such risks were disclosed to the subjects, and whether Institutional Review Board (IRB) approvals (or an equivalent approval/review based on the requirements of your country or institution) were obtained?
    \item[] Answer: \answerNA{} 
    \item[] Justification: No human subjects research occurred in this work.

\item {\bf Declaration of LLM usage}
    \item[] Question: Does the paper describe the usage of LLMs if it is an important, original, or non-standard component of the core methods in this research? Note that if the LLM is used only for writing, editing, or formatting purposes and does \emph{not} impact the core methodology, scientific rigor, or originality of the research, declaration is not required.
    \item[] Answer: \answerNA{} 
    \item[] Justification: LLM is not part of our core method.

\end{enumerate}

\end{document}